\documentclass[twoside]{article}

\usepackage[accepted]{aistats2023}
\usepackage[utf8]{inputenc} 
\usepackage[T1]{fontenc}    
\usepackage{hyperref}       
\usepackage{url}            
\usepackage{booktabs}       
\usepackage{amsfonts}       
\usepackage{nicefrac}       
\usepackage{microtype}      
\usepackage{xcolor}         
\usepackage{graphicx}
\usepackage{subcaption}
\usepackage{amssymb, amsmath, amsthm, paralist,color}
\usepackage{enumitem}
\usepackage{dsfont}
\usepackage{multirow}
\usepackage{comment}
\usepackage{caption}
\usepackage{float}
\floatstyle{plaintop}
\usepackage{natbib}
\restylefloat{table}
\usepackage{dsfont}

\newtheorem{theorem}{Theorem}
\newtheorem{corollary}{Corollary}
\newtheorem{definition}{Definition}

\newcommand{\bbE}{\mathbb{E}}
\newcommand{\bbR}{\mathbb{R}}
\newcommand{\bbP}{\mathbb{P}}

\newcommand{\be}{\mathbf{e}}
\newcommand{\bx}{\mathbf{x}}
\newcommand{\by}{\mathbf{y}}

\newcommand{\calR}{\mathcal{R}}

\newcommand{\calA}{\mathcal{A}}

\newcommand{\calD}{\mathcal{D}}

\newcommand{\calK}{\mathcal{K}}
\newcommand{\calM}{\mathcal{M}}
\newcommand{\calN}{\mathcal{N}}

\newcommand{\calY}{\mathcal{Y}}
\newcommand{\calX}{\mathcal{X}}

\newcommand*\diff{\mathop{}\!\mathrm{d}}

\DeclareMathOperator*{\argmax}{arg\,max}

\newtheorem{manualtheoreminner}{Theorem}
\newenvironment{manualtheorem}[1]{%
  \renewcommand\themanualtheoreminner{#1}%
  \manualtheoreminner
}{\endmanualtheoreminner}

\begin{document}

%

%

\twocolumn[

\aistatstitle{Does Label Differential Privacy Prevent Label Inference Attacks?}

\aistatsauthor{ Ruihan Wu$^*$ \And Jin Peng Zhou$^*$ \And  Kilian Q. Weiberger \And Chuan Guo }

\aistatsaddress{ Cornell University \And  Cornell University \And Cornell University \And Meta AI} ]

%

\begin{abstract}
    Label differential privacy (label-DP) is a popular framework for training private ML models on datasets with public features and sensitive private labels. Despite its rigorous privacy guarantee, it has been observed that in practice label-DP \emph{does not} preclude label inference attacks (LIAs): Models trained with label-DP can be evaluated on the public training features to recover, with high accuracy,  the very private labels that it was designed to protect. 
In this work, we argue that this phenomenon is not paradoxical and that label-DP is designed to limit the \emph{advantage} of an LIA adversary compared to predicting training labels using the \emph{Bayes classifier}. At label-DP $\epsilon=0$ this advantage is zero, hence the optimal attack is to predict according to the Bayes classifier and is independent of the training labels. 
Our bound shows the semantic protection conferred by label-DP and gives guidelines on how to choose $\varepsilon$ to limit the threat of LIAs below a certain level.
Finally, we empirically demonstrate that our result closely captures the behavior of simulated attacks on both synthetic and real world datasets.
\end{abstract}
\section{INTRODUCTION}
Differential privacy (DP)~\citep{dwork2006calibrating, dwork2014algorithmic} has become the foundational tool for private learning on sensitive training data.
More recently, this framework has been adopted for training \emph{label differentially private} (label-DP) models~\citep{pmlr-v19-chaudhuri11a, ghazi2021deep, esmaeili2021antipodes}, where only the label of a training sample is considered sensitive and must be protected. One prominent application for label-DP is online advertisement, where the learning goal is to predict whether a user clicked on an ad or not, which is a private and sensitive label, given the product description for the displayed ad.

Intuitively, label-DP presents an easier task for the learner compared to DP since the training features are assumed to be public. Indeed, prior work showed that label-DP learning algorithms can achieve much higher test accuracy compared to the best DP counterparts on benchmark datasets. However, such models also attain a high accuracy on the training set, which enables an adversary to simply evaluate the model on the public training features to (correctly) predict the private labels~\citep{busa2021pitfalls}---a method that we refer to as the \emph{simple prediction attack} (SPA). The existence of such a paradoxical adversary raises the question of whether label-DP is truly a meaningful privacy notion to strive for.

In this paper, we take a closer look at the connection between label-DP and label inference attacks (LIAs).
We first show that label-DP is unable to upper bound the accuracy of LIAs under arbitrarily small values of the privacy parameter $\epsilon$.
This limitation applies not only to label-DP, but \emph{any} model that generalizes will inevitably enable the SPA attack to attain high label inference accuracy.
In the extreme case where the learning algorithm perfectly generalizes, the output model becomes the Bayes classifier and the SPA attack's accuracy is determined entirely by the Bayes error rate, which is independent of the training labels.

Our analysis suggests that it is unreasonable to equate label privacy with limiting the accuracy of LIAs in absolute terms. At a high level, such an argument is in-line with the design principle of DP as not to protect against statistical inference~\citep{mcsherry2016statistical, bun2021statistical}. Instead, we consider the \emph{advantage} of an LIA adversary over predicting training labels according to the Bayes classifier. Such advantage can only have originated from memorizing the training set and therefore leakage of private labels, and 
vice versa an adversary with zero advantage is no better than the Bayes classifier that is completely independent of the training labels. Under this analytical framework, we show that an $\varepsilon$-label-DP learner can reduce this advantage to $1- \frac{2}{1 + e^{\varepsilon}}$.
Importantly, our bound shows that at low $\varepsilon$, even if an label-DP learner achieves high training accuracy, it does not necessarily reveal any sensitive information about the training labels---resolving the aforementioned paradox.
Our bound gives semantic meaning to the label-DP $\varepsilon$ and can be used as a guideline for calibrating the value of $\varepsilon$ for practical use cases.

We empirically validate the advantage upper bound on both simulated and real world datasets. On the simulated dataset where the Bayes classifier is known, our upper bound dominates advantage of the SPA attack and is fairly tight at both small and large $\varepsilon$ values. We also evaluate on the Criteo 1TB Click Logs dataset~\citep{tallis2018reacting}, which closely resembles the learning setting in common applications of label-DP where the ground truth label is very noisy and the marginal label distribution is highly imbalanced. Our result shows that advantage of the SPA attack becomes negative at even moderate values of the privacy parameter $\varepsilon$ despite the attack attaining close to $97\%$ label inference accuracy.

\section{PRELIMINARIES}
\label{sec:preliminaries}
\paragraph{Notations.} Let $\calX, \calY$ denote the feature and label space, respectively, and let $D = (X, \by) \in (\calX \times \calY)^n$ be a training dataset consisting of $n$ training samples.
Let $\calD$ be the underlying data distribution.
For $i=1,\dots,n$, $X_{-i}\in \calX^{n-1}$ denotes the training features except for the $i$th sample.

\textbf{Differential privacy} 
(DP)~\citep{dwork2006calibrating, dwork2014algorithmic} is a standard tool for privacy-preserving data analysis that hides the contribution of any individual training sample to the mechanism's output. In the context of machine learning, this is achieved by randomizing the learner's output and requiring that replacing one data point by another does not lead to a significant change in the output distribution. We restate its formal definition below.

\begin{definition}[$(\varepsilon,\delta)$-Differential Privacy]
	Let $\varepsilon, \delta\in \bbR^{\geq0}$. A randomized training algorithm $\calM: (\calX\times\calY)^n \to \cal R$ with domain $(\calX\times\calY)^n$ and range $\calR$ satisfies $(\varepsilon,\delta)$-differential privacy if for any two adjacent datasets $D, D'\in(\calX\times\calY)^n$, 
    which differ at exactly one data point $(\bx, y)$, 
	and for any subset of outputs $S\subseteq \calR$, it holds that:
	$$
	\bbP[\calM(D)\in S] \leq e^{\varepsilon}\cdot \bbP[\calM(D')\in S] + \delta.
	$$
\end{definition}

\textbf{Label differential privacy} 
(label-DP)~\citep{pmlr-v19-chaudhuri11a} is a relaxation of DP where only the privacy of training \emph{labels} must be protected. 
This setting assumes that the training features are public and/or non-sensitive but the labels are sensitive and are kept secret. Such a scenario arises naturally in several common applications of ML:

1. In online advertising, ads are selected by an ML model for display to maximize click-through rate (CTR)---the percentage of users that will click on the ad~\citep{richardson2007predicting, mcmahan2013ad, chapelle2014simple}. The model is trained on features such as product and advertiser description, and a binary label of whether the user clicked on a displayed ad or not. In this application, features are publicly accessible and non-sensitive, but the label indicates user interest and is considered sensitive and private.

2. In recommendation systems, the learning goal is to suggest products or webpages to a user based on features such as user profile, search query, and descriptions of products/webpages, which are available to the recommender~\citep{ricci2011introduction}. The training labels are historical data of user rating or click and are considered private.

The existence of a label-only privacy setting motivates the study of label-DP.
Different from DP, the notion of adjacency applies only to the label of a single training sample: $D$ and $D'$ are identical except for one data point $(\bx, y) \in D$ and $(\bx, y') \in D'$; see below for a formal definition.

\begin{definition}[$(\epsilon,\delta)$-Label Differential Privacy]
\label{def:label_dp}
Let $\varepsilon, \delta\in \bbR^{\geq0}$. A randomized training algorithm $\calM$ taking as input a dataset is said to be $(\varepsilon, \delta)$-label differentially private ($(\varepsilon, \delta)$-label-DP) if for any two training datasets $D$ and $D'$ that differ in the label of a single example, and for any subset $S$ of outputs of $\calM$, it holds that 
$$\bbP(\calM(D)\in S)\leq e^{\varepsilon}\cdot \bbP(\calM(D')\in S) + \delta.$$
When $\delta=0$, we simply refer to $\calM$ as $\epsilon$-label-DP.
\end{definition}

\paragraph{Label-DP learning algorithms.}
The first mechanism for achieving label differential privacy is \emph{randomized response} (RR)~\citep{warner1965randomized}, which (with a certain probability) randomly samples training labels according to a pre-determined distribution before releasing them to the learner. Recent works proposed several label-DP learning algorithms that are inspired by RR:

1. \emph{Label Private Multi-Stage Training} (LP-MST; \citep{ghazi2021deep}) randomly samples training labels $\by_i$ using a learned prior sampling distribution $\bbP(y|X_i)$ instead of the pre-determined distribution in RR. Such a prior could be learned by observing the top-K predictions using a pre-trained model and limiting RR to this subset of most likely labels. An alternative way is to divide the training process into multiple stages and leverage the model trained in the previous stage as the prior for predicting the top-K labels.

2. \emph{Private Aggregation of Teacher Ensembles with FixMatch} (PATE-FM; \citep{esmaeili2021antipodes}) uses FixMatch~\citep{sohn2020fixmatch}---a semi-supervised learning algorithm---to train several teacher models for private aggregation. Each teacher is trained on all training features together with a subset of revealed labels, with this subset disjoint among different teachers.
Finally, a student model is trained using PATE~\citep{papernot2016semi} to predict differentially privately aggregated labels from the teachers' predictions given public training features.

3. \emph{Additive Laplace Noise Coupled with Bayesian Inference} (ALIBI; \citep{esmaeili2021antipodes}) releases differentially private training labels by perturbing one-hot encodings of the labels using the Laplace mechanism~\citep{ghosh2012universally}.
Since post-processing preserves differential privacy~\citep{dwork2014algorithmic}, the resulting noisy labels can then be denoised using Bayesian inference to maximize the probability of recovering the clean label.
\section{DOES LABEL-DP PREVENT LABEL INFERENCE ATTACKS?}
\label{sec:setup}

\begin{table*}[t]
    \centering
    \resizebox{\linewidth}{!}{
    \begin{tabular}{c|cccc|cccc}
    \toprule
       \multirow{2}{*}{\textbf{Algorithm}} & \multicolumn{2}{c}{\textbf{MNIST} ($\varepsilon=1.0$)} & \multicolumn{2}{c|}{\textbf{MNIST} ($\varepsilon=0.1$)} & \multicolumn{2}{c}{\textbf{CIFAR10} ($\varepsilon=1.0$)} & \multicolumn{2}{c}{\textbf{CIFAR10} ($\varepsilon=0.1$)} \\
    \cmidrule{2-9}
        & Test Acc. & Attack Acc. & Test Acc. & Attack Acc. & Test Acc. & Attack Acc. & Test Acc. & Attack Acc. \\
    \midrule
       LP-1ST & 93.3 & 93.3 & 20.8 & 20.9 & 61.5 & 61.9 & 15.5 & 15.8\\
       LP-1ST ({\small in-domain prior})  & 97.1 & 96.5 & 97.0 & 96.2 & 75.4 & 75.7 & 66.3 & 66.3 \\
       LP-1ST ({\small out-of-domain pior}) & 94.6 & 93.7 & 86.2 & 85.2  & 89.5 & 89.8 & 87.6 & 86.9 \\
       PATE-FM & 99.3 & 99.1 & 23.6 & 23.0 & 92.4 & 92.1 & 18.6 & 18.6 \\
       ALIBI & 96.3 & 96.3 & 21.5 & 20.8 & 67.5 & 69.6 & 13.6 & 13.9 \\
    \bottomrule
    \end{tabular}
    }
    \caption{Model test accuracy and attack accuracy of the simple prediction attack (SPA) evaluated on label-DP models trained on MNIST and CIFAR10. SPA attack accuracy is equivalent to training accuracy for classification and is exceptionally high in most cases. Our evaluation shows that a learning algorithm can offer a very stringent label-DP guarantee of $\varepsilon=0.1$ while failing to prevent label inference attacks.}
    \label{tab:empi_eval_ldp_algo}
\end{table*}

Relaxing DP to label-DP provided the flexibility for designing more specialized private learning algorithms. These methods seem to provide excellent trade-offs between privacy and model utility, as measured by their high test accuracy even at very low $\varepsilon$, $\delta$ values.
In this section, we take a closer look at the privacy protection offered by label-DP. We argue that not only is label-DP unable to prevent adversaries from inferring the training labels under arbitrarily low values of privacy parameter $\varepsilon > 0$, \emph{any} model that generalizes will inevitably fail to do so as well.

\subsection{Label Inference Attack against Label-DP}
\label{sec:paradox_label_dp}

\paragraph{Label Inference Attack in Vertical Federated Learning.} In the setting of \textit{Vertical Federated Learning} (VFL), one party owns the features and the other party owns the labels. The objective is to jointly train a model without leaking private information about the labels. One way to achieve this is using split learning, where the party with features holds the first several layers of the model and the party with labels holds the remaining layers of the model, and the exchange of gradients at the split layer is protected by label-DP. However, \citet{sun2022label} showed that even when the label-DP $\varepsilon$ is as small as $0.5$, the party with the features can still infer the private labels with a prediction AUC of $0.75$.

\paragraph{Label Inference Attack in Label-DP Model Training.} We first make the observation that the label-DP guarantee \emph{does not} imply that an adversary cannot leverage the model to infer its training labels. In fact, if the training accuracy is high, the adversary can trivially evaluate the model on the public training set and recover its labels~\citep{busa2021pitfalls}. We refer to this attack as the \emph{simple prediction attack} (SPA) and evaluate it on existing label-DP learning algorithms.

\autoref{tab:empi_eval_ldp_algo} shows test accuracy for several label-DP models trained on MNIST~\citep{lecun1998gradient} and CIFAR10~\citep{krizhevsky2009learning}, along with the corresponding SPA attack accuracy, \emph{i.e.}, the model's training accuracy. There is a clear trend that these algorithms can achieve very high test accuracies with strong label-DP guarantee, \emph{i.e.}, low privacy parameter $\varepsilon$.
For instance, at $\varepsilon=0.1$, the test accuracy could reach as high as $97.0$ for MNIST and $87.6$ for CIFAR10.
However, the SPA attack accuracy is almost identical to the model's test accuracy, which (paradoxically) seems to suggest that models leak a tremendous amount of label information even when trained with stringent label-DP guarantees.

\subsection{Impossibility of Label Protection under Label-DP}
\label{sec:impossibility_ldp}

Following the above observation, a natural question to ask is whether the vulnerability of label-DP to the existing label inference attack is due to an insufficiently strong privacy guarantee, \emph{i.e.}, $\varepsilon,\delta$ being not small enough. We give a definitive negative answer by formalizing \emph{label inference attacks} (LIAs) and showing that label-DP cannot guarantee protection against LIAs even for arbitrarily small values of $\varepsilon,\delta > 0$.

\paragraph{Threat model.}
The adversary's goal is to design a label inference attack algorithm $\calA$ that infers the training label of each sample in the training dataset.
The output of $\calA$ is a vector of inferred labels $\hat{\by} \in \calY^{n}$. We assume that the adversary has access to the following information:
\begin{enumerate}[leftmargin=*,nosep]
    \item The adversary has full knowledge of the output $o = \calM(X, \by)$, where $\calM$ is any releasing algorithm. The output $o$ could be the model when we consider the label-DP model training. It could also be a sequential of message passed between parties when label-DP is applied in the federated learning setting.
    \item The adversary has full knowledge of the feature matrix $X\in\bbR^{n\times d}$.
    \item (Optional) The adversary has knowledge of the conditional data distribution $\bbP(y|\bx)$ of $\calD$.
\end{enumerate}

We refer to the threat model with or without the third assumption as the \emph{with-prior} or \emph{priorless} setting. The with-prior setting is not unrealistic: Given access to a separate dataset for the same learning task, the adversary can train a shadow model to estimate the conditional probability $\bbP(y|\bx)$. Such an approach is commonly used in membership inference attacks~\citep{shokri2017membership}.

\paragraph{Expected attack utility.} To measure how successful an LIA is at inferring training labels, we define the \emph{attack utility function} $u(\hat{\by}_i, \by_i)$ where $\hat{\by}_i \in \calY$ is the inferred label and $\by_i$ is the ground truth. We assume that $u(\hat{\by}_i, \by_i) \in [0,B]$ for any $\hat{\by}_i, \by_i \in \calY$,
\emph{e.g.}, $u(\hat{\by}_i, \by_i)=\mathds{1}(\hat{\by}_i = \by_i)$ is the zero-one accuracy for classification problems, which is bounded with $B=1$. For regression problems with bounded label range $\calY\subseteq [-b, b]$, we can define $u(\hat{\by}_i, \by_i) = 4b^2 - (\hat{\by}_i - \by_i)^2$, which is bounded with $B=4b^2$. We assume $B=1$ for simplicity; our results can be easily generalized to any $B > 0$.

The \emph{expected attack utility} (EAU) is defined as the expectation of $u(\hat{\by}_i, \by_i)$ over the randomness of the sampling of labels and in the learning algorithm:
\begin{equation}
\label{equ:adv}
    \mathrm{EAU}(\calA, \calK) = \bbE_{\by,\calM}\left[ \left. \frac{1}{n} \sum_{i=1}^n u(\hat{\by}_i, \by_i ) \right| X\right],
\end{equation}
where $\calK$ denotes the adversary's knowledge: $(X, o, \bbP(y|\bx))$ for the with-prior setting and $(X, o)$ for the priorless setting.

\paragraph{Upper bound on expected attack utility.}
For attack utility functions $u$ that reflect the accuracy of a label inference attack, one may ask whether label-DP can provide a uniform upper bound $U(\varepsilon, \delta)$ such that
\begin{equation}
\label{equ:uniform_bound}
    \mathrm{EAU}\left(\calA, \calK\right) \leq U(\varepsilon, \delta)
\end{equation}
holds for any data distribution $\calD$, feature matrix $X$ and attack algorithm $\calA$.
A trivial upper bound $U(\varepsilon, \delta) \leq 1$ follows from the boundedness of $u$.
Unfortunately, we show that this bound is in fact optimal for both the with-prior and the priorless settings.

\begin{theorem}
\label{thm:upper_bound}
There is no function $U(\varepsilon,\delta)$ that satisfies \autoref{equ:uniform_bound} and is strictly less than $1$ at some $\varepsilon, \delta>0$.
\end{theorem}
\begin{proof}
We first consider the with-prior setting where the adversary has access to the conditional distribution $\bbP(y|\bx)$. For a classification problem with utility function $u(\hat{\by}_i, \by_i)=\mathds{1}(\hat{\by}_i = \by_i)$, we define the attack $\calA_{\rm prior}$ that predicts training labels according to the Bayes classifier:
\begin{equation}
\label{equ:prior_att}
\calA_{\rm prior}(\calK) := \left(\argmax_{y\in\calY}\bbP(\by_i=y|X_i): i=1, \cdots n\right).
\end{equation}
The expected attack utility for $\calA_{\rm prior}$ is:
\begin{equation}
\label{equ:prior_adv}
    \mathrm{EAU}\left(\calA_{\rm prior}, \calK\right) = \frac{1}{n}\sum_{i=1}^n\max_{y\in\calY}\bbP(\by_i=y|X_i).
\end{equation}
In particular, when the label $y$ is deterministic given the feature $\bx$, \emph{i.e.}, $\bbP(y|\bx) = 1$ for some $y \in \calY$, this EAU evaluates to $1$. Note that $\calA_{\rm prior}$ does not depend on $o$ and is thus valid for any $\varepsilon, \delta$, hence $U(\varepsilon, \delta)= 1 ~\forall \varepsilon, \delta>0$.

For the \emph{priorless} setting, consider again a classification problem with the same utility function $u$. Denote by $\calA_{\rm priorless}$ the simple prediction attack, which predicts labels using the label-DP trained model $f$.
We will construct a series of settings ($\calD^n$, $X^n$, $f^n$) such that each $f^n$ is $(\varepsilon, \delta)$-label-DP and as $n\to \infty$, ${\rm \text{EAU}}\left(\calA_{\rm priorless}, \calK\right) \to 1$, which shows that $U(\varepsilon, \delta)\geq 1 ~\forall \varepsilon, \delta>0$.
\begin{itemize}[leftmargin=*,nosep]
    \item \emph{Data construction}: The feature domain $\calX$ is $\{-1, 1\}$ and the label space $\calY$ is $\{-1, 1\}$. We construct $X^n$ with $n = 2r$ samples where $X_1^n= \cdots = X_r^n = 1$ and $X_{r+1}^n = \cdots = X_{2r}^n = -1$. The conditional label distribution given the feature is $\bbP(\by_i^n = X_i^n | X_i^n) = 1$ for all $i$.
    \item \emph{Label-DP algorithm for training $f^n$}: We apply randomized response with $\bbP(\tilde{\by}_i = \by_i^n) = \frac{e^{\varepsilon}}{e^{\varepsilon} + 1}$ to privatize the labels. We then train $f^n$ on the privatized labels to maximize training accuracy, which results in $f^n$ being simply the majority sign function: $f^n(1) = \text{sign}\{ \sum_{i=1}^r \tilde{\by}_i \}$ and $f^n(-1) = \text{sign}\left\{ \sum_{i=r+1}^{2r} \tilde{\by}_i \right\}$.
\end{itemize}
The fact that $f^n$ is $(\varepsilon, 0)$-label-DP follows from $\tilde{\by}_i$ being the randomized response of $\by_i$ for $i=1,\ldots,n$ and $f^n$ being a post-processing function. The SPA attack's EAU is equal to the training accuracy of the model $f^n$, which can be lower bounded using Hoeffding's inequality:
\begin{align*}
    &\hspace{3ex} \bbE\left[ \left. \frac{1}{n} \sum_{i=1}^n \mathds{1}\left[\calA_{\rm priorless}(X^n, f^n)_i = \by_i^n\right] \right| X^n \right] \\
    &= \frac{1}{2} \cdot \bbP\left( \sum_{i=1}^r \tilde{\by}_i >0\right) + \frac{1}{2} \cdot \bbP\left( \sum_{i=r+1}^{2r} \tilde{\by}_i <0\right)\\
    &\geq  \bbP\left( \left|\frac{1}{r}\sum_{i=1}^r \frac{1 + \tilde{\by}_i}{2} - \frac{e^{\varepsilon}}{e^{\varepsilon} + 1}\right| < \frac{e^{\varepsilon}}{e^{\varepsilon} + 1}  - \frac{1}{2}  \right) \\
    &\geq 1 - 2\exp\left(-\left(\frac{e^{\varepsilon}}{e^{\varepsilon} + 1} - \frac{1}{2}\right)^2 n\right).
\end{align*}
Taking $n \rightarrow \infty$ gives the desired result.
\end{proof}

\autoref{thm:upper_bound} shows that for both the with-prior and priorless settings, no non-trivial upper bound for EAU exists for any label-DP privacy parameters $\varepsilon,\delta>0$. It also validates our rationalization about the experimental result in \autoref{tab:empi_eval_ldp_algo} that failure to prevent the SPA attack is to be expected.

\section{LABEL-DP PROVABLY BOUNDS ATTACK ADVANTAGE}
\label{sec:label_dp}
The impossibility results derived in \autoref{sec:setup} suggest that limiting the EAU of label inference attacks may not be a reasonable objective for label privacy. In particular, since the with-prior attack $\calA_\text{prior}$ in \autoref{thm:upper_bound} completely disregards the trained model $o$, it should be treated as a baseline for measuring the effectiveness of LIAs. 
Indeed, $\calA_\text{prior}$ generalizes the Bayes classifier and is optimal among attacks that are independent of the training labels $\by$ by construction. Hence any attack that achieves an EAU equal to or less than that of $\calA_\text{prior}$ \emph{does not gain any additional information} about the training labels from $o$. Thus, we refer to
$$
\mathrm{EAU}(\calA_\text{prior}, \calK) = 
\frac{1}{n}\sum_{i=1}^n\max_{y\in \calY}\bbE_{\by_i} \left[u(y, \by_i)|X_i\right]
$$
as the \emph{label-independent expected attack utility} (L-EAU), and instead measure the success of a label inference attack $\calA$ by defining its \emph{advantage}:
\begin{equation}
    \label{equ:excess_adv}
    \mathrm{Adv}\left(\calA, \calK\right) := \mathrm{EAU}\left(\calA, \calK\right) - \mathrm{EAU}(\calA_\text{prior}, \calK).
\end{equation}

Next, we show that label-DP can effectively reduce the advantage of LIAs to close to $0$ when its privacy parameters $\varepsilon,\delta$ are small. We first prove a distribution-dependent upper bound in
\autoref{thm:excess_adv_upper} that holds for any label-DP output $o$ and attack algorithm $\calA$ but depends on the conditional distribution $\bbP(y | \bx)$, and then give a universal upper bound $\mathrm{Adv}\left(\calA, \calK\right) \leq U(\varepsilon, \delta)$ in Corollary \ref{cor:univ_excess_adv_upper} that only depends on the label-DP parameters $(\varepsilon, \delta)$. Proof is given in \autoref{app:derivations}.

\begin{theorem}
\label{thm:excess_adv_upper}
Assume each label $\by_i$ is sampled independent of $(\by_{-i}, X_{-i})$. If $o$ satisfies $(\varepsilon, \delta)$-label-DP then for any attack algorithm $\calA$, we have:
\begin{align*}
    &\mathrm{Adv}\left(\calA, \calK\right) \\
    &\leq \left(1 - \frac{2}{1 + e^{\varepsilon}} (1 - \delta)\right)\cdot \left(\frac{1}{n}\sum_{i=1}^n\bbE_{\by_i|X_i}\left[\sup_{y\in\calY}u(y, \by_i)\right] \right).
\end{align*}
\end{theorem}

\begin{proof}
First note that the adversary's inferred label vector $\hat{\by} = \calA(X, \calM(X, \by))$ is a random variable that depends on both the sampling of training labels $\by$ and randomness in the learning algorithm $\calM$. Then:
\begin{align}
\label{equ:total_expectation}
    \bbE [u(\hat{\by}_i, \by_i) | X, \by_{-i}] &= \bbE_{\by_i|X_i}\left[\bbE [u(\hat{\by}_i, \by_i) |\by, X]\right],
\end{align}
where the equality holds by the assumption that $\by_i$ is independent of $X_{-i}$ and $\by_{-i}$. For each $i=1, \cdots, n$, let $B(\by_i) = \sup_{y\in\calY}u(y, \by_i)$ be the maximal attack utility attainable when inferring the ground truth label $\by_i$. Consider an alternative label vector $\by'$ that is identical to $\by$ except for $\by'_i$ being replaced with some deterministic label value $y^*$, and denote by $\hat{\by}' = \calA(X, \calM(X, \by'))$ the adversary's inferred labels for the model trained on $(X, \by')$.
By label-DP, we have:
    \begin{align*}
        \bbE [u(\hat{\by}_i, \by_i) |\by, X]
        &=\int_0^{B(\by_i)} \bbP(u(\hat{\by}_i, \by_i) > v |\by, X) \diff v\\
        &\leq\int_0^{B(\by_i)} \bbP(u(\hat{\by}_i', \by_i) > v |\by, X) \diff v \\
         &+\int_0^{B(\by_i)} \left(1 - \frac{2}{1 + e^{\varepsilon}} (1 - \delta)\right)\diff v\\
        &= \bbE [u\left(\hat{\by}_i', \by_i\right) |\by, X] \\
        &+  \left(1 - \frac{2}{1 + e^{\varepsilon}} (1 - \delta)\right)\cdot B(\by_i),
    \end{align*}
where the inequality follows the Remark A.1 in \cite{kairouz2015composition} Substituting the above inequality into \autoref{equ:total_expectation} gives:
\resizebox{\linewidth}{!}{
  \begin{minipage}{\linewidth}
\begin{align*}
\bbE [u(\hat{\by}_i, \by_i) | X, \by_{-i}]
&\leq \bbE\left[u\left(\hat{\by}_i', \by_i\right)|X, \by_{-i}\right] \\
&+  \left(1 - \frac{2}{1 + e^{\varepsilon}} (1 - \delta)\right)\cdot \bbE_{\by_i|X_i}\left[B(\by_i)\right]\\
&\leq \max_{y\in\calY} \bbE_{\by_i|X_i}[u(y, \by_i)] \\
&+  \left(1 - \frac{2}{1 + e^{\varepsilon}} (1 - \delta)\right)\cdot \bbE_{\by_i|X_i}\left[B(\by_i)\right],
\end{align*}
  \end{minipage}
}
where the last inequality holds by the fact that $\hat{\by}_i'$ is independent of $\by_i$ conditioned on $X$ and $\by_{-i}$.
Finally, we can derive our bound for the advantage $\mathrm{Adv}\left(\calA, \calK\right)$:\\
\resizebox{\linewidth}{!}{
  \begin{minipage}{\linewidth}
\begin{align*}
 &\mathrm{EAU}\left(\calA, \calK\right) 
    - \frac{1}{n}\sum_{i=1}^n\max_{y\in \calY}\bbE_{\by_i} \left[u(y, \by_i)|X_i\right]\\
 &=\frac{1}{n}\sum_{i=1}^n\left(\bbE_{\by_{-i}}\bbE [u(\hat{\by}_i, \by_i) | X, \by_{-i}] - \max_{y\in \calY}\bbE_{\by_i} \left[u(y, \by_i)|X_i\right]\right)\\
 &\leq   \left(1 - \frac{2}{1 + e^{\varepsilon}} (1 - \delta) \right)\cdot \left(\frac{1}{n}\sum_{i=1}^n\bbE_{\by_i|X_i}\left[\sup_{y\in\calY}u(y, \by_i)\right] \right).
\end{align*}
  \end{minipage}
}
\end{proof}

\begin{corollary}
\label{cor:univ_excess_adv_upper}
Suppose $u(y', y) \in [0, B]$ for any $y', y\in\calY$ and each label $\by_i$ is sampled independent of $(\by_{-i}, X_{-i})$. If $o$ satisfies $(\varepsilon, \delta)$-label-DP then for any data distribution $\calD$, any feature matrix $X$ and any attack algorithm $\calA$, we have:
\begin{align*}
    &\mathrm{Adv}\left(\calA, \calK\right) \leq \left(1 - \frac{2}{1 + e^{\varepsilon}} (1 - \delta)\right)\cdot B.
\end{align*}
\end{corollary}

\paragraph{Interpretation of \autoref{thm:excess_adv_upper}.} 
We can interpret \autoref{thm:excess_adv_upper} by revisiting the example in \autoref{sec:paradox_label_dp}.
Instead of bounding the EAU of a label inference attack, \autoref{thm:excess_adv_upper} shows that label-DP with low $(\varepsilon,\delta)$ can upper bound the \emph{advantage} to close to $0$. 
This result explains why even with the strong guarantee of $\epsilon$-label-DP at $\varepsilon=0.1$ in \autoref{tab:empi_eval_ldp_algo}, the attack utility could still be as high as $80\%+$:
Because both MNIST and CIFAR10 admit a high L-EAU (\emph{i.e.}, high accuracy of the Bayes classifier), LIAs that attain $80\%+$ EAU may not even outperform the label-independent attack $\calA_\text{prior}$, hence models trained by label-DP algorithms do not leak a significant amount of information about training labels.

Moreover, we observe that the bound for the advantage, or equivalently the EAU, is relative to both the label-DP parameter $\varepsilon$ and the underlying distribution. 
When $\frac{1}{n}\sum_{i=1}^n\bbE_{\by_i|X_i}\left[\sup_{y\in\calY}u(y, \by_i)\right]$ is higher, a higher $\varepsilon$ is sufficient to achieve the same level of protection against LIAs.
This interpretation is in-line with the design principle of DP, which is meant to limit the difference between the prior and posterior distributions for the underlying data~\citep{dwork2014algorithmic, kasiviswanathan2014semantics}.
From a practical aspect, one can use Corollary \ref{cor:univ_excess_adv_upper} to calibrate the values of $\varepsilon$ and $\delta$ for the dataset at hand to limit the utility of arbitrary label inference attack.

\subsection{Label-DP vs. DP}

Remarkably, using DP even when only the label is private can give stronger semantic guarantees against LIAs than the label-DP guarantee in \autoref{thm:excess_adv_upper}. This is however not true under the threat model defined in \autoref{sec:impossibility_ldp} where the feature matrix $X$ is public, but holds under a weaker threat model where the feature matrix is non-private but \emph{unknown}. This threat model has been considered in \citet{ghazi2021deep} and was implicitly used to motivate the randomized response mechanism~\citep{warner1965randomized}.

In essence, with $X$ unknown, the with-prior attack $\calA_\text{prior}$ is no longer viable. Instead, the optimal attack without observing the trained model is to guess according to the \emph{marginal distribution} of $\by_i$ for each $i$, resulting in an EAU of $\max_{y\in\calY}\bbE_{\by_i}[u(y, \by_i)]$, which is provably smaller than the L-EAU when the feature matrix is known and we denote it as L-EAU$^w$.
For example, when we consider the data from CIFAR10, the L-EAU is able to attain $80\%+$ when feature matrix is known, while L-EAU$^w$ is only $10\%$ in the weaker threat model without the knowledge of the feature matrix.
We define the new advantage corresponding to the weaker threat model as 
\begin{equation}
    \label{equ:excess_adv_featknown}
    \mathrm{Adv}^{w}\left(\calA, \calK^{w}\right) := \mathrm{EAU}\left(\calA, \calK^{w}\right) - \text{L-EAU}^w,
\end{equation}
where $\calK^w=o$ denotes the adversary's knowledge in this weaker setting.

Due to the lack of protection for the feature matrix, label-DP is not capable limit this advantage to $0$.
This is intuitive: with a successful feature (data) reconstruction attack~\citep{fredrikson2015model, carlini2019secret, zhu2019deep}, the adversary will have the knowledge of feature and hence achieve previous higher L-EAU.
Instead, DP including the protection of features can successfully limit this advantage of the weaker threat model into 0.
\autoref{thm:excess_adv_upper_dp} below gives a precise statement; see \autoref{app:derivations} for proof.
\begin{theorem}
\label{thm:excess_adv_upper_dp}
Assume each data $(X_i, \by_i)$ is sampled independently. If $f$ satisfies $(\varepsilon, \delta)$-DP then for any attack algorithm $\calA$, we have:
\begin{align*}
    &\mathrm{Adv}^{w}\left(\calA, \calK^{w}\right) \\
    &\leq \left(1 - \frac{2}{1 + e^{\varepsilon}} (1 - \delta)\right)\cdot \left(\frac{1}{n}\sum_{i=1}^n\bbE_{\by_i}\left[\sup_{y\in\calY}u(y, \by_i)\right] \right).
\end{align*}
\end{theorem}
Hence, in the scenario where the adversary does not have access to the feature matrix $X$, DP gives a stronger guarantee against label inference attacks and can be preferred over label-DP especially when L-EAU is relatively large.
\section{EXPERIMENTS}
\begin{figure*}[t]%
\includegraphics[width=\linewidth]{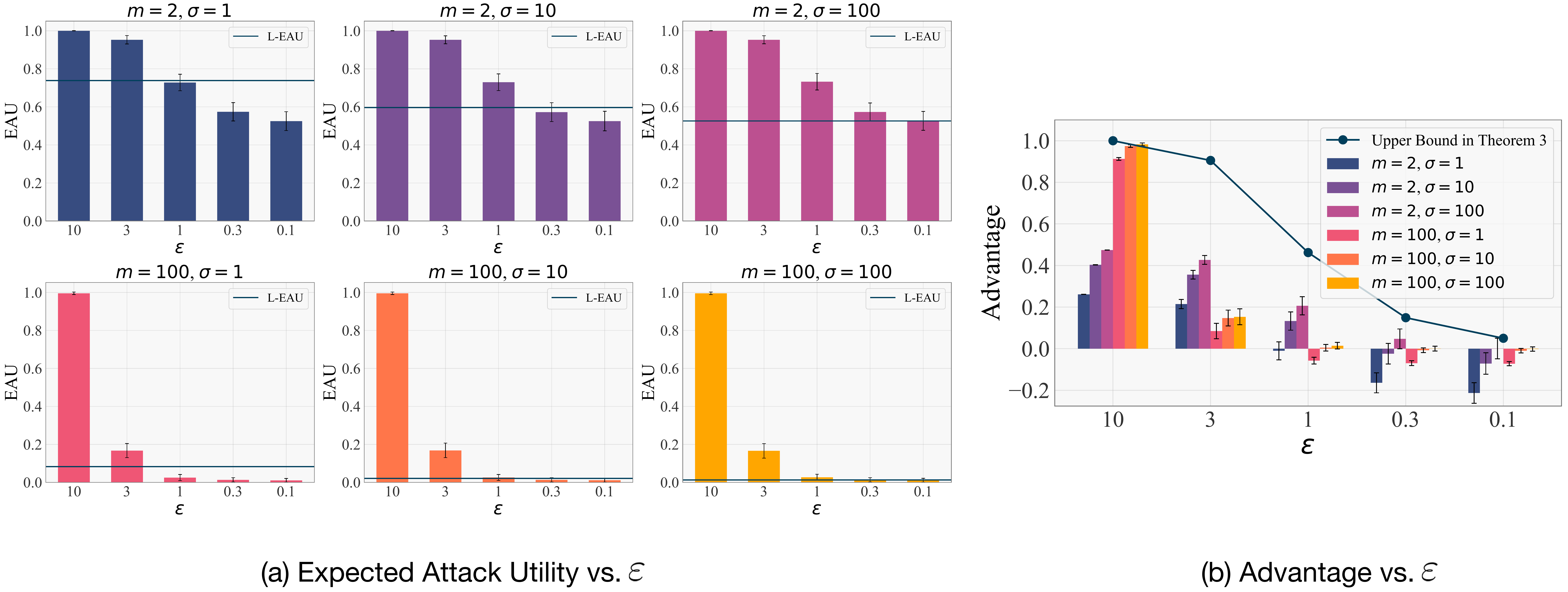}
\caption{
\textbf{(a)} EAU and L-EAU vs. $\varepsilon$ on the simulated dataset with $m=2,100$ and $\sigma=1, 10, 100$. EAU decreases with lower $\varepsilon$ (stronger privacy guarantee) and can be lower than L-EAU for moderately low values of $\varepsilon$. \textbf{(b)} Attack advantage and theoretical upper bound vs. $\varepsilon$. The upper bound dominates the advantage while being fairly tight for the end values $\varepsilon=0.1$ and $\varepsilon=10$ for $m=100$. The error bar represents standard deviation across $T=1000$ different draws of label vector $\by$.}
\label{fig:simulation}
\end{figure*}
We validate \autoref{thm:excess_adv_upper} on both a simulated Gaussian mixture dataset (\autoref{sec:simulated_exp}) and a real world ads click prediction dataset (\autoref{sec:real_exp}) and show that empirical results obey the theoretically derived upper bound.
For full reproducibility, we release our code at \url{https://github.com/jinpz/label_differential_privacy/}.

\subsection{Simulated Dataset with Gaussian Mixture}
\label{sec:simulated_exp}

\paragraph{Data generation.}
We define a classification setting where the feature space $\calX = \mathbb{R}^d$ with $d=100$ and the label space $\calY = \{1,\ldots,m\}$ for $m=2$ or $m=100$ classes. For each class $i$, features are sampled from an isotropic Gaussian $\calN(\be_i, \sigma^2 I_d)$ where $\be_i$ is the standard basis vector. We vary $\sigma \in \{1,10,100\}$ and the resulting data distribution $\calD$ is the uniform mixture of the $m$ classes' distributions.

\paragraph{Model training.} To train a private model that satisfies label-DP, we first draw $n=100$ random samples from the mixture distribution $\calD$. Given a target label-DP privacy parameter $\epsilon$, the learning algorithm trains a logistic regressor by generating privatized labels using randomized response~\citep{warner1965randomized} to satisfy $\epsilon$-label-DP. This process is repeated multiple times to estimate the EAU of the simple prediction attack; see the following paragraph for details.

\paragraph{Attack evaluation.} We evaluate the simple prediction attack (SPA) using the utility function $u(\hat{y}, y)=\mathds{1}\{\hat{y}=y\}$. To estimate the expected attack utility (EAU), we first \emph{fix} a random draw of training features $X$ from the marginal distribution of $\calD$. Given $X$ and the Gaussian mixture parameters, we can compute the conditional probability of each $y \in \calY$ using mixture densities and sample the label vector $\by$ accordingly. This process is repeated $T=1000$ times for the same training features to estimate the expectation over $\by$ in \autoref{equ:adv}. We find that the sampling of $X$ does not significantly impact our result. Finally, we can construct the Bayes classifier by choosing the label $y$ for each $X_i$ that maximizes the conditional probability $\bbP(y|X_i)$ and compute L-EAU directly.

\paragraph{Results.} 
In \autoref{fig:simulation}(a), we plot EAU and L-EAU against $\varepsilon$ for the simulated dataset with $m = 2,100$ and $\sigma=1, 10, 100$. For both settings, EAU is close to 100\% when $\varepsilon=10$, and decreasing $\varepsilon$ (\emph{i.e.}, stronger privacy guarantee) leads to smaller EAU values. For $m=2$, although EAU is consistently high even at $\epsilon=0.1$, L-EAU is also very high, hence the attack does not attain a very large advantage (difference between EAU and L-EAU). For $m=100$, L-EAU is much lower due to the classification problem being harder, while EAU also drops very quickly as $\varepsilon$ decreases. 

In \autoref{fig:simulation}(b), we plot advantage and the upper bound in \autoref{thm:excess_adv_upper} against $\varepsilon$. As expected, the upper bound dominates attack advantage in all settings, and both values decrease monotonically as $\varepsilon$ decreases.
The upper bound is tight at the two end values $\varepsilon=0.1$ and $\varepsilon=10$ for $m=100$, whereas for $m=2$ it can be fairly loose.
This is because L-EAU for $m=2$ is at least $0.5$ and therefore advantage is always upper bounded by $0.5$. Our upper bound could potentially be improved if the minimum L-EAU value can be inferred from the task and/or data distribution.

\begin{table*}[t]
    \centering
    \begin{tabular}{c|cccccc}
        \toprule
         \textbf{Label-DP Algorithm} & $\varepsilon=\infty$ & $\varepsilon=8$ & $\varepsilon=4$ & $\varepsilon=2$ & $\varepsilon=1$ & $\varepsilon=0.1$\\
        \midrule 
        LP-1ST  & 0.123 &  0.123 &	0.129 &	0.204 &	0.360 &	0.651 \\
        LP-1ST ({\small domain prior})  & 0.123 & 0.123 & 0.128 & - & - & - \\
        LP-1ST ({\small noise correction}) & 0.123 & 0.123 &	0.126 &	0.151 &	0.177 &	0.646 \\
        LP-2ST  & 0.123 & 0.123	& 0.129 &	0.202 &	0.346 &	0.530 \\
        PATE  & 0.130 & 0.151 & 0.164 & 0.194 & 0.248 & 0.676 \\
        \bottomrule
    \end{tabular}
    \caption{Log loss of state-of-the-art label-DP algorithms under different $\varepsilon$. When $\varepsilon\leq 2$, none of the label-DP algorithms outperform the constant prediction baseline, which attains a log loss of $0.135$. For LP-1ST ({\small domain prior}) at $\varepsilon\in\{2, 1, 0.1\}$, the mechanism heavily relied on the prior and returned label $0$ with probability $1$ for all training samples, so training did not yield any meaningful result.}
    \label{tab:real_world_results}
\end{table*}

\subsection{Criteo Ads Click-Through Rate (CTR) Prediction}
\label{sec:real_exp}

To understand the behavior of LIAs and our advantage bound on real world datasets that closely resemble the usage scenarios of label-DP, we conduct experiments on the Criteo 1TB Click Logs dataset\footnote{\href{https://ailab.criteo.com/download-criteo-1tb-click-logs-dataset}{https://ailab.criteo.com/download-criteo-1tb-click-logs-dataset}}~\citep{tallis2018reacting} for click-through rate (CTR) prediction. 

\begin{figure*}
    \centering
    \includegraphics[width=\linewidth]{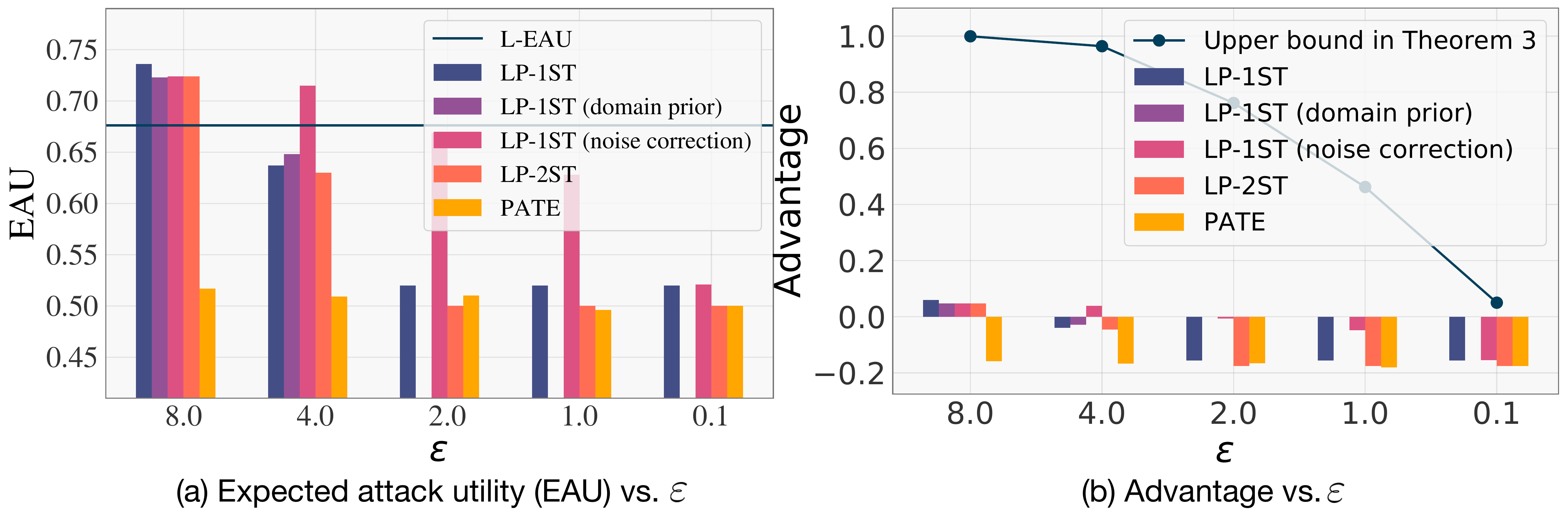}
    \caption{Expected attack utility (EAU) and advantage of the simple prediction attack against label-DP models trained on the Criteo dataset. L-EAU of $0.676$ is estimated using maximum test accuracy achieved on this dataset among all trained models. Even at $\epsilon=8$, the EAU for label-DP models is marginally above L-EAU, hence the corresponding advantage is only slightly above $0$. In (b), we show that \autoref{thm:excess_adv_upper} strictly upper bounds attack advantage but there exists a large gap compared to the computed advantage.}
    \label{fig:criteo_attack}
\end{figure*}

\paragraph{Dataset description.}
In CTR prediction, the task is to predict the percentage of users that will click on the ad given ad features.
The features consist of 13 integer values and 26 categorical features, while the semantic of these features is undisclosed. The binary label indicates whether a user clicked on the ad or not. The marginal label distribution is heavily skewed with approximately $97\%$ of samples having the label $0$, \emph{i.e.}, no click.

The dataset contains more than 4B click log data points spanning across 24 days of data collection.
We subsample 1M data points from the first day's entries, and take $80\%$ as the training set, $4\%$ as the validation set and the remaining $16\%$ as the test set.
Following the Kaggle competition\footnote{\href{https://www.kaggle.com/c/criteo-display-ad-challenge}{https://www.kaggle.com/c/criteo-display-ad-challenge}} for this dataset,
we evaluate model utility using log loss:
$$
\frac{1}{|D_\text{test}|}\sum_{(\bx, y)\in D_\text{test}}-y\cdot \log\left(f(\bx)\right) + (y-1) \cdot \log\left(1-f(\bx)\right).$$

\paragraph{Model training.}
We implemented gradient-boosted decision tree~\citep{friedman2001greedy} in CatBoost~\citep{dorogush2018catboost} as the baseline non-private learning algorithm. We further adapted LP-MST~\citep{ghazi2021deep} and PATE~\citep{papernot2016semi} to this setting as label-DP learning algorithms; see \autoref{app:exp} for implementation details. For LP-MST, we considered multiple variants:
LP-1ST, LP-1ST ({\small domain prior}), LP-1ST ({\small noise correction}) and LP-2ST.

\paragraph{Attack evaluation.}
Since the dataset is heavily skewed towards the label $0$, simply predicting the all-zero label vector $\hat{\by}$ can attain an approximately $97\%$ attack accuracy under the zero-one accuracy utility function. To correct for this bias, we consider a \emph{weighted} attack utility function:
\begin{equation}
    \label{equ:reweighting}
    u(\hat{y}, y):=\frac{1}{2p_y}\mathds{1}\{\hat{y} = y\},
\end{equation}
where $p_y$ is the marginal probability of label $y$. This re-weighting has a meaningful interpretation for the adversary as well: The label $1$ represents a user click and is more valuable to infer compared to the label $0$ that represents no click, and hence should be given a higher utility when predicted correctly. Under this attack utility, predicting all-$0$, all-$1$ or randomly guessing $y$ with probability $p_y$ all attain an EAU of $1/2$. Finally, we adapt the SPA attack accordingly to optimize re-weighted utility:
 \resizebox{\linewidth}{!}{
   \begin{minipage}{\linewidth}
    \begin{align*}
        \calA_{\rm priorless}(X, f)_i
        = \argmax_{y\in\{0, 1\}}\frac{1}{p_y}\cdot (y\cdot f(X_i) + (1-y)\cdot (1-f(X_i)).
    \end{align*}
   \end{minipage}
 }

\paragraph{Computing EAU.} Since the conditional label distribution is unknown, we cannot compute EAU or L-EAU directly as in the simulated dataset experiment in \autoref{sec:simulated_exp}. Instead, we use a model's training accuracy (weighted according to \autoref{equ:reweighting}) as an unbiased estimator for the EAU of the SPA attack. For L-EAU, any ML model evaluated on the test set gives a valid lower bound, and we pick the maximum over all models trained on this dataset as an approximation for L-EAU.

\paragraph{Results.}
\autoref{tab:real_world_results} shows the model utility of label-DP models trained with $\varepsilon\in\{\infty, 8, 4, 2, 1, 0.1\}$.
Since the dataset is heavily skewed with $p_0 \approx 0.97$ fraction of samples belonging to class $0$, the constant predictor $f(\bx) = p_0$ for all $\bx$ achieves a log loss of $0.135$.
In comparison, the label-DP algorithms fail to outperform this naive baseline when $\varepsilon\leq 2$. Our evaluation suggests that there is much room for improvement in existing label-DP learning algorithms for this highly noisy and heavily skewed learning setting.

Next, we plot EAU of the SPA attack along with our estimate of L-EAU in \autoref{fig:criteo_attack}(a). 
The maximum attainable EAU is $1.0$, while both EAU of the SPA attack and the estimated L-EAU
are not very high, which reflects the noisy nature of this dataset. Moreover, EAU quickly deteriorates to below L-EAU when $\varepsilon=2$, despite the attack accurately inferring approximately $97\%$ of the training labels by predicting (nearly) all-zero.
Finally, we see in \autoref{fig:criteo_attack}(b) that advantage of these attacks is very low and can be negative for $\varepsilon \leq 2$. We also evaluate the upper bound in \autoref{thm:excess_adv_upper}, where the quantity $\bbE_{\by_i|X_i}\left[\sup_{y\in\calY}u(y, \by_i)\right] = 1$ by construction of $u$ (\emph{cf.} \autoref{equ:reweighting}). Although this bound again dominates the computed advantage, there exists a very large gap. We suspect this is due to a number of reasons, including the SPA attack being sub-optimal or due to looseness in label-DP accounting. Future work may be able to design better LIAs that exploit model memorization in other ways or use tighter label-DP accounting to reduce this gap.

\section{DISCUSSION AND CONCLUSION}
\label{sec:conclusion}

\citet{busa2021pitfalls} first noted the fact that even with the label-DP guarantee, an adversary can still recover the training labels via the simple prediction attack. Their Bayesian interpretation of this paradox is that the public release of features contributed to this privacy loss and that label-DP cannot mitigate this risk. We offer a different view and advocate that label leakage in absolute terms should not be considered a privacy violation. Instead, with the appropriate metric of success for the adversary, \emph{i.e.}, advantage, we showed that label-DP can indeed prevent label inference attacks. We hope that future work can build upon our analysis to deepen our understanding of the connection between label-DP and LIAs. Lastly, we do not foresee any negative societal impacts for our work.

\paragraph{Limitations.} Our paper focuses on deriving LIA advantage bounds for $(\epsilon,\delta)$-label-DP. Alternative formulations of DP such as R\'{e}nyi-DP~\citep{mironov2017renyi} and Gaussian DP~\citep{dong2019gaussian} offer much tighter privacy accounting for composition of mechanisms, and hence it may be of interest to derive similar bounds under these accounting frameworks. Moreover, our evaluation on the Criteo dataset is only a preliminary study. Follow-up work on thoroughly analyzing label-DP algorithms and studying their limitations on such challenging datasets is warranted for a better understanding of the privacy-utility trade-offs.

\subsubsection*{Acknowledgements}
RW, JPZ and KQW are supported by grants from the National Science Foundation NSF (IIS-2107161, III-1526012, IIS-1149882, and IIS-1724282), and the Cornell Center for Materials Research with funding from the NSF MRSEC program (DMR-1719875), and SAP America.


\renewcommand{\bibsection}{\subsubsection*{References}}
\bibliography{main}
\bibliographystyle{apalike}
%
\newpage
\appendix
\onecolumn
\section{Proof of Theorem 3}
\label{app:derivations}

\begin{manualtheorem}{3}
Assume each data $(X_i, \by_i)$ is sampled independently. If $f$ satisfies $(\varepsilon, \delta)$-DP then for any attack algorithm $\calA$, we have:
\begin{equation*}
    \mathrm{Adv}^w\left(\calA, \calK\right) \leq \left(1 - \frac{2}{1 + e^{\varepsilon}} (1 - \delta)\right)\cdot \left(\frac{1}{n}\sum_{i=1}^n\bbE_{\by_i}\left[\sup_{y\in\calY}u(y, \by_i)\right] \right).
\end{equation*}
\end{manualtheorem}

\begin{proof}
The proof follows similarly from proof of Theorem 3. First note that the adversary's inferred label vector $\hat{\by} = \calA(\calM(X, \by))$ is a random variable that depends on both the sampling of training labels $\by$ and randomness in the learning algorithm $\calM$. Then:
\begin{align}
\label{equ:total_expectation_dp}
    \bbE [u(\hat{\by}_i, \by_i) | X_{-i}, \by_{-i}] &= \bbE_{\by_i, X_i}\left[\bbE [u(\hat{\by}_i, \by_i) |\by, X]\right],
\end{align}
where the equality holds by the assumption that $\by_i$ is independent of $X_{-i}$ and $\by_{-i}$. 
By DP, we have:
    \begin{align*}
        \bbE [u(\hat{\by}_i, \by_i) |\by, X]
        &=\int_0^{B(\by_i)} \bbP(u(\hat{\by}_i, \by_i) > v |\by, X) \diff v\\
        &\leq\int_0^{B(\by_i)} \left(\bbP(u(\hat{\by}_i', \by_i) > v |\by, X) + \left(1 - \frac{2}{1 + e^{\varepsilon}} (1 - \delta)\right)\right) \diff v\\
        &= \bbE [u\left(\hat{\by}_i', \by_i\right) |\by, X] +  \left(1 - \frac{2}{1 + e^{\varepsilon}} (1 - \delta)\right)\cdot B(\by_i),
    \end{align*}
where the inequality once again follows the Remark A.1 in \cite{kairouz2015composition}. Substituting the above inequality into \autoref{equ:total_expectation_dp} gives:
\begin{align*}
\bbE [u(\hat{\by}_i, \by_i) | X_{-i}, \by_{-i}]
&\leq \bbE\left[u\left(\hat{\by}_i', \by_i\right)|X_{-i}, \by_{-i}\right] +  \left(1 - \frac{2}{1 + e^{\varepsilon}} (1 - \delta)\right)\cdot \bbE_{\by_i, X_i}\left[B(\by_i)\right]\\
&\leq \max_{y\in\calY} \bbE_{\by_i, X_i}[u(y, \by_i)] +  \left(1 - \frac{2}{1 + e^{\varepsilon}} (1 - \delta)\right)\cdot \bbE_{\by_i, X_i}\left[B(\by_i)\right],
\end{align*}
where the last inequality holds by the fact that $\hat{\by}_i'$ is independent of $\by_i$ conditioned on $X_{-i}$ and $\by_{-i}$.
Finally, we can derive our bound for the advantage $\mathrm{Adv}\left(\calA, \calK\right)$:\\
\resizebox{\linewidth}{!}{
  \begin{minipage}{\linewidth}
\begin{align*}
 \mathrm{EAU}\left(\calA, \calK\right) 
    - \frac{1}{n}\sum_{i=1}^n\max_{y\in \calY}\bbE_{\by_i} \left[u(y, \by_i)\right]
 &=\frac{1}{n}\sum_{i=1}^n\left(\bbE_{\by_{-i}}\bbE [u(\hat{\by}_i, \by_i) | X_{-i}, \by_{-i}] - \max_{y\in \calY}\bbE_{\by_i} \left[u(y, \by_i)\right]\right)\\
 &\leq   \left(1 - \frac{2}{1 + e^{\varepsilon}} (1 - \delta) \right)\cdot \left(\frac{1}{n}\sum_{i=1}^n\bbE_{\by_i}\left[\sup_{y\in\calY}u(y, \by_i)\right] \right).
\end{align*}
  \end{minipage}
}
\end{proof}

\section{Experiment Details}\label{app:exp}

The implementation details of label-DP algorithms for Criteo CTR prediction dataset are listed as below.
\begin{enumerate}[leftmargin=*,nosep]
    \item LP-1ST (RR): LP-1ST is a one-stage version of LP-MST. Notice that it is equivalent to Randomized Response (RR)~\cite{warner1965randomized}, which flips each training label to other labels uniformly to satisfy the label-DP.
    \item LP-1ST ({\small domain prior}): We follow the Algorithm 4 in \cite{ghazi2021deep} to compute the prior label distribution for each data and feed these prior distributions to LP-1ST. We set the number of clusters as 100. Before the clustering, we encode all categorical features into one-hot representations and normalize integer features into the range $[0, 1]$.
    \item LP-1ST ({\small noise correction}): LP-1ST injects uniform noise to the training labels before the training of gradient boost. We can additionally adapt a post-training noise correction method in \cite{zhang2021learning} for LP-1ST to achieve better performance.
    \item LP-2ST: LP-2ST is the two-stage version of LP-MST.
    \item PATE: 
    We make two adaptions for the original PATE to keep the information of those minority labels of $1$:
    \begin{itemize}[leftmargin=*,nosep]
        \item When a trained teacher predicts the label for a data point, instead of outputting the label that has maximum probability score, we sample a label from the output probability vector.
        \item When we aggregate the prediction from each teacher for student's training, instead of outputting the label with maximum count in the noisy histogram of teacher's predictions, we again sample a label from the probability, normalized from the noisy prediction histogram.
    \end{itemize}
    Without the above two adaptions, as we empirically verified, all aggregated labels would be 0 and the student model would be meaningless. Moreover, we perform data-independent privacy cost accounting following~\cite{papernot2016semi} and obtain different $\varepsilon$ by varying number of queries with fixed noise level.
\end{enumerate}

\end{document}